\documentclass[11pt]{article}
\usepackage[margin=1in]{geometry}
\usepackage{amsmath,amssymb,amsthm,mathtools,booktabs,array,hyperref,graphicx}
\hypersetup{hidelinks}
\usepackage{enumitem}
\usepackage[T1]{fontenc}
\usepackage[utf8]{inputenc}

\newtheorem{theorem}{Theorem}
\newtheorem{lemma}{Lemma}
\newtheorem{proposition}{Proposition}

\theoremstyle{remark}
\newtheorem{remark}{Remark}

\newcommand{\R}{\mathbb R}
\newcommand{\E}{\mathbb E}
\newcommand{\Prob}{\mathbb P}
\newcommand{\co}{\operatorname{co}}
\newcommand{\dist}{\operatorname{dist}}
\newcommand{\interior}{\operatorname{int}}
\newcommand{\one}{\mathbf 1}
\DeclareMathOperator{\diag}{diag}
\DeclareMathOperator{\sgn}{sgn}

\title{Scalar-Stepsize Nonuniform Monte Carlo Optimistic Policy Iteration: A Certified Counterexample}
\author{Yuanlong Chen}
\date{}

\begin{document}
\maketitle

\begin{abstract}
Tsitsiklis~\cite{Tsitsiklis2002} proved convergence of Monte Carlo optimistic policy iteration under a uniform update structure and identified nonuniform update frequencies as a delicate obstruction.  We give a certified
negative answer for the natural scalar-stepsize, unnormalized asynchronous state-value recursion with fixed nonuniform state-selection probabilities. In a three-state, two-action discounted MDP, the nonuniform update frequencies induce a diagonally scaled greedy-policy mean field with a certified nonconstant
attracting hybrid periodic orbit.  With a bounded unbiased geometric-horizon estimator and Robbins--Monro stepsizes, the original stochastic recursion remains trapped near the cycle with positive probability and therefore fails to converge. The example pinpoints a geometric obstruction: uniform sampling gives radial residual contraction, whereas scalar nonuniform sampling anisotropically distorts the residual dynamics and can generate switched attracting cycles.
\end{abstract}

\section{Introduction}\label{sec:introduction}

Optimistic policy iteration (OPI), also known in related Monte Carlo forms as Monte Carlo exploring starts, alternates between partial policy evaluation and greedy policy improvement; see Sutton and Barto \cite{SuttonBarto2018} for the standard Monte Carlo control formulation.  In a finite discounted Markov decision process (MDP), the fully synchronous dynamic-programming counterpart is well behaved, but the Monte Carlo optimistic variant is more delicate because policy evaluation is incomplete and the greedy policy can change while the value estimates are still moving.

A foundational convergence result of Tsitsiklis \cite{Tsitsiklis2002} established convergence for a special discounted Monte Carlo OPI scheme under a uniform update structure.  The same analysis also shows why nonuniform update frequencies are delicate: the scalar commutativity used in the residual proof is lost unless the sampling imbalance is removed by an appropriate component-dependent normalization.  This paper addresses the unnormalized scalar-stepsize version of that nonuniform regime for the asynchronous state-value recursion
\begin{equation}\label{eq:intro-recursion}
    J_{t+1}=J_t+\gamma_t e_{I_t}\bigl(R_t-J_t(I_t)\bigr),
\end{equation}
where only one coordinate is updated at each iteration, $\Prob(I_t=i)=q_i$, and the stepsize $\gamma_t$ is a single scalar shared by all coordinates.

The answer is negative.  We construct a three-state, two-action discounted MDP, together with a nonuniform update distribution \(q\) for which the recursion \eqref{eq:intro-recursion} fails to converge with positive probability.  The mechanism is visible already at the mean-field level.  If $\mu(J)$ is the greedy policy at $J$, the expected drift away from policy ties is
\begin{equation}\label{eq:intro-drift}
    \dot J=D\bigl(J^{\mu(J)}-J\bigr),
    \qquad D=\diag(q_1,\ldots,q_n).
\end{equation}
When $D$ is a scalar multiple of the identity, each fixed-policy region has a straight-line relaxation structure in suitable residual variables.  When $D$ is nonuniform, this scalar structure is replaced by a diagonal distortion.  In the example below, that distortion creates an attracting switched cycle of the greedy-policy mean field.

The main mathematical difficulty is that $J\mapsto \mu(J)$ is discontinuous.  The proof therefore treats the limiting dynamics as a differential inclusion on policy-tie surfaces, certifies that the candidate orbit crosses all switching surfaces one-way, and verifies a positive-radius Poincare contraction on an actual two-dimensional return section.  The stochastic recursion is then shown to stay near this attracting cycle on a positive-probability small-noise event.

\paragraph{Contributions.}
This paper gives a certified negative answer to the nonuniform-state-selection question for the natural scalar-stepsize, unnormalized, state-value version of asynchronous Monte Carlo OPI. We construct a three-state, two-action discounted MDP and a fixed nonuniform update distribution for which the original stochastic recursion fails to converge with positive probability.  The example identifies the failure mechanism: scalar nonuniform sampling replaces the uniform scalar relaxation by a diagonally scaled greedy-policy mean field, and this distortion can create a stable nonoptimal cycle rather than convergence to the Bellman fixed point.  The computer-assisted orbit certificate and the martingale trapping argument are used to make this counterexample fully rigorous.

\paragraph{Organization.}
Section~\ref{sec:literature} reviews the relevant literature.  Section~\ref{sec:algorithm} defines the recursion and the limiting mean field.  Section~\ref{sec:main-result} states the certified counterexample and the main nonconvergence theorem.  Section~\ref{sec:certificate-summary} summarizes the finite certificates, while the numerical details are moved to Appendices~\ref{app:data} and~\ref{app:certificates}.  Section~\ref{sec:stochastic-lift} proves the stochastic positive-probability conclusion, with the detailed hybrid perturbation argument placed in Appendix~\ref{app:hybrid-proof}.  Section~\ref{sec:optimality-mechanism} identifies the optimal value, mechanism, and separation from the certified cycle.  Section~\ref{sec:discussion} discusses the convergence/nonconvergence boundary and open problems.  Appendix~\ref{app:scripts} describes the support verification scripts.

\section{Literature review and relation to prior work}\label{sec:literature}

\paragraph{Optimistic policy iteration and Monte Carlo exploring starts.}
The starting point is Tsitsiklis \cite{Tsitsiklis2002}, which proves convergence of a special Monte Carlo optimistic policy iteration algorithm in the discounted finite-state setting.  The proof relies on a uniform update structure.  In residual variables, uniform sampling makes the limiting fixed-policy dynamics essentially scalar, and this scalar structure is central to the argument.  Tsitsiklis also explains that nonuniform initial-state sampling can be normalized by component-dependent empirical stepsizes, but that the unnormalized scalar update loses the commutativity behind the proof.  The present work gives a negative answer for this scalar-stepsize asynchronous state-value formulation.

Several later papers study related Monte Carlo OPI or MCES convergence questions. Chen \cite{Chen2018} and Liu \cite{Liu2021} investigate convergence in stochastic shortest path settings.  Liu's analysis also emphasizes that, in nonuniform settings, component-dependent normalization is the natural interpretation of Monte Carlo exploring starts.  Other positive results impose structural or parameter restrictions, such as feed-forward or acyclic transient structure \cite{WangYuanShaoRoss2022,LubarsWinnickiLivesaySrikant2021} or a small discount factor for the first-visit Monte Carlo algorithm \cite{DelattreFournier2026}.  More recently, Oliviers and Vinnicombe \cite{OliviersVinnicombe2026} prove a positive result beyond uniform state-action
updates for an initial-visit MC-O-PI scheme that retains uniformity over actions within each state. Their result concerns an action-value MC-O-PI formulation with within-state action-uniform updates, whereas the recursion studied here is the scalar state-value update \eqref{eq:intro-recursion}.  The distinction is essential: their within-state action-uniform normalization removes the unnormalized scalar nonuniform regime.  Their analysis also emphasizes that the classical commutativity argument no longer applies when states are updated at different frequencies.  The present counterexample exhibits the corresponding failure mode in a state-value scalar-stepsize setting: nonuniform state-coordinate frequencies without component-dependent normalization create a diagonally distorted greedy-policy mean field with a certified attracting cycle.

\paragraph{Relation to approximate-policy-iteration oscillations.}
The cycle certified here is different from residual-method and approximate-policy-iteration pathologies that can occur with function approximation, projection, or natural actor-critic style updates \cite{Baird1995,Wagner2011,Wagner2013}.  No function approximation, projection, or sampling approximation of the MDP model is used in the deterministic certificate.  The instability is caused instead by the diagonal update-frequency factor in an exact tabular mean field.

\paragraph{Asynchronous stochastic approximation.}
Asynchronous stochastic approximation provides the natural language for update-frequency effects.  Borkar \cite{Borkar1998} studies asynchronous stochastic approximations and shows how limiting differential equations are modified by relative update frequencies.  In the present setting, this modification appears as the diagonal matrix $D=\diag(q_1,\ldots,q_n)$ in \eqref{eq:intro-drift}.  The key point is that this diagonal factor is harmless in some settings but can be dynamically destructive in scalar-stepsize nonuniform OPI.

\paragraph{Differential inclusions and positive-probability attraction.}
Because greedy policies can be nonunique on tie surfaces, the closed limiting object is a differential inclusion rather than a globally Lipschitz ODE.  The general stochastic-approximation/differential-inclusion framework is developed by Benaim, Hofbauer, and Sorin \cite{BenaimHofbauerSorin2005}; Perkins and Leslie \cite{PerkinsLeslie2012} extend this perspective to asynchronous stochastic approximation with set-valued mean fields.  Faure and Roth \cite{FaureRoth2010} prove a general positive-probability convergence-to-attractor theorem for stochastic approximations governed by set-valued dynamical systems.  Our stochastic lift is consistent with this literature.  The main new work here is not a new abstract stochastic-approximation theorem, but rather the certified construction of a non-equilibrium attractor in a concrete OPI mean field, together with a local trapping proof for the original Monte Carlo recursion.

\section{Scalar-Stepsize Nonuniform MC-OPI and Its Mean Field}\label{sec:algorithm}


Let
\[
\mathcal M=(S,A,P,g,\alpha)
\]
be a finite discounted Markov decision process.  Here
\(S=\{1,\ldots,n\}\) is the finite state space, \(A\) is the finite action
space, \(P(i,a,j)\) is the probability of moving from state \(i\) to state
\(j\) after choosing action \(a\), \(g(i,a)\) is the one-stage cost, and
\(0<\alpha<1\) is the discount factor.  A deterministic stationary policy is
a map \(\mu:S\to A\).  Under policy \(\mu\), the controlled Markov chain evolves
according to \(P_\mu(i,j)=P(i,\mu(i),j)\), and its discounted cost-to-go is
\[
J^\mu(i)
=
\E_i^\mu\left[\sum_{t=0}^{\infty}
\alpha^t g(S_t,\mu(S_t))\right].
\]
Since the model is finite and discounted, \(J^\mu\) is the unique solution of
the linear Bellman equation
\[
J^\mu = g_\mu+\alpha P_\mu J^\mu,
\]
or equivalently
\[
J^\mu=(I-\alpha P_\mu)^{-1}g_\mu .
\]

For a deterministic policy \(\mu\), write
\[
(T_\mu J)(i)=g(i,\mu(i))+\alpha\sum_j P(i,\mu(i),j)J(j),
\]
and define the optimal Bellman operator by
\[
(TJ)(i)=\min_{a\in A}\left\{g(i,a)+\alpha\sum_jP(i,a,j)J(j)\right\}.
\]
The fixed point of \(T_\mu\) is \(J^\mu\), while the unique fixed point of
\(T\) is the optimal value \(J^*\).  A policy \(\mu\) is greedy for \(J\) if
\[
T_\mu J=TJ.
\]
Let \(G(J)\) be the finite set of greedy policies at \(J\). Exact policy iteration alternates between greedy improvement and complete policy
evaluation: from a value estimate \(J\), choose a greedy policy
\(\mu\in G(J)\), compute its exact value \(J^\mu\), and then improve again.
Optimistic policy iteration (OPI) does not wait for complete evaluation before
improving the policy.  Instead, after choosing a greedy policy from the current
estimate, it takes only a partial evaluation step toward \(J^\mu\) and then
recomputes the greedy policy at the next iterate.  Thus policy improvement and
policy evaluation are interleaved.

The scalar asynchronous Monte Carlo OPI recursion is
\begin{equation}\label{eq:async-recursion}
    J_{t+1}=J_t+\gamma_t e_{I_t}\bigl(R_t-J_t(I_t)\bigr),
\end{equation}
where $I_t$ is the single coordinate updated at time $t$, $\Prob(I_t=i)=q_i$, and $\gamma_t$ is deterministic and scalar.  We fix a measurable tie-breaking rule $\tau$ for greedy policies, for instance lexicographic tie-breaking, and set $\mu_t=\tau(J_t)\in G(J_t)$.

Throughout the stochastic recursion, let $(\mathcal F_t)$ be the natural filtration generated by the initial condition and all samples before the $t$th update. Thus $J_t$ and $\mu_t$ are $\mathcal F_t$-measurable. Conditional on $\mathcal F_t$, the coordinate $I_t$ is sampled with the fixed law $q$ and is independent of the geometric horizon and Markov-chain randomness used to generate the return. Conditional on $(\mathcal F_t,I_t=i)$, the return $R_t$ is generated from starting state $i$ under policy $\mu_t$. In the stochastic theorem we use
\[
    \gamma_t=\frac{1}{t+N}
\]
with $N$ sufficiently large.

For the counterexample, we choose $n=3$, $A=\{0,1\}$, and
\[
    q=(0.02,0.18,0.80),
    \qquad D=\diag(q_1,q_2,q_3).
\]
The full cost and transition data are listed in Appendix~\ref{app:data}.  For $J\in\R^3$, define
\[
    Q_i(a;J)=g(i,a)+\alpha P(i,a)J
\]
and
\[
    d_i(J)=Q_i(1;J)-Q_i(0;J).
\]
Since costs are minimized,
\[
    d_i(J)>0 \Longleftrightarrow \text{action }0\text{ is uniquely greedy at state }i,
\]
and
\[
    d_i(J)<0 \Longleftrightarrow \text{action }1\text{ is uniquely greedy at state }i.
\]

\subsection{Sampling model}

We use a bounded unbiased geometric-horizon Monte Carlo estimator.  Given a starting state $i$ and a policy $\mu$, sample
\[
    T\sim\operatorname{Geom}(1-\alpha)-1,
    \qquad \Prob(T=m)=(1-\alpha)\alpha^m,
    \quad m=0,1,2,\ldots .
\]
Run the Markov chain under $\mu$ for $T$ transitions from $S_0=i$ and define
\begin{equation}\label{eq:geometric-estimator}
    R=\frac{g(S_T,\mu(S_T))}{1-\alpha}.
\end{equation}

\begin{lemma}[Unbiasedness and boundedness]\label{lem:geom-estimator}
For every deterministic policy $\mu$ and starting state $i$,
\[
    \E[R\mid S_0=i,\mu]=J^\mu(i).
\]
Moreover, if $G_{\max}=\max_{i,a}|g(i,a)|$, then
\[
    |R|\le \frac{G_{\max}}{1-\alpha}\quad\text{almost surely}.
\]
\end{lemma}

\begin{proof}
Conditioning on the random horizon and terminal state gives
\[
\begin{aligned}
\E[R\mid S_0=i,\mu]
&=\sum_{m=0}^{\infty}(1-\alpha)\alpha^m
   \sum_{j=1}^3(P_\mu^m)(i,j)\frac{g(j,\mu(j))}{1-\alpha} \\
&=\sum_{m=0}^{\infty}\alpha^m(P_\mu^m g_\mu)(i)
=\bigl[(I-\alpha P_\mu)^{-1}g_\mu\bigr](i)
=J^\mu(i).
\end{aligned}
\]
The bound follows immediately from \eqref{eq:geometric-estimator}.
\end{proof}

Because $\mu_t$ is $\mathcal F_t$-measurable and the coordinate and return sampling are conditionally independent as specified above, Lemma~\ref{lem:geom-estimator} gives
\[
    \E[R_t\mid\mathcal F_t,I_t=i]=J^{\mu_t}(i).
\]
Therefore
\begin{equation}\label{eq:mean-drift}
\begin{aligned}
\E[J_{t+1}-J_t\mid\mathcal F_t]
&=\gamma_t\sum_{i=1}^3q_i e_i\bigl(J^{\mu_t}(i)-J_t(i)\bigr)\\
&=\gamma_t D\bigl(J^{\mu_t}-J_t\bigr).
\end{aligned}
\end{equation}
Away from tie surfaces, $G(J)$ is a singleton and the mean ODE is
\begin{equation}\label{eq:mean-ode}
    \dot J=D\bigl(J^{\mu(J)}-J\bigr).
\end{equation}
At ties the closed mean field is the differential inclusion
\begin{equation}\label{eq:DI}
    \dot J\in H(J),
    \qquad
    H(J)=D\,\co\{J^\mu-J:\mu\in G(J)\}.
\end{equation}
The map $H$ is upper semicontinuous with nonempty compact convex values because there are finitely many policies.

Inside a fixed-policy region, $\mu(J)=\mu$, the vector field is affine:
\[
    f_\mu(J)=D(J^\mu-J).
\]
For the concrete update frequencies $q=0.02(1,9,40)$, the exact flow can be written by putting $z=e^{-0.02t}$:
\begin{equation}\label{eq:Phi}
    \Phi_\mu(x,z)=J^\mu+(x-J^\mu)\odot(z,z^9,z^{40}).
\end{equation}

\section{Main result}\label{sec:main-result}

The counterexample follows the policy sequence
\begin{equation}\label{eq:policy-sequence}
(0,1,1)\to(0,0,1)\to(1,0,1)\to(1,0,0)\to(1,1,0)\to(1,1,1)\to(0,1,1).
\end{equation}
Let
\[
\mu_0=(0,1,1),\quad
\mu_1=(0,0,1),\quad
\mu_2=(1,0,1),
\]
\[
\mu_3=(1,0,0),\quad
\mu_4=(1,1,0),\quad
\mu_5=(1,1,1).
\]
The finite certificates in Appendix~\ref{app:certificates} verify the following proposition.

\begin{proposition}[Certified attracting hybrid cycle]\label{prop:certified-cycle}
For the MDP in Appendix~\ref{app:data} and the update distribution $q=(0.02,0.18,0.80)$, the hybrid ODE \eqref{eq:mean-ode} has a nonconstant locally attracting periodic orbit $C$ following the policy sequence \eqref{eq:policy-sequence}.  Moreover, $C$ is a strict local attractor for the local Filippov differential inclusion \eqref{eq:DI}.
\end{proposition}

\begin{proof}
The proof is certified by C1--C4+ in Appendix~\ref{app:certificates}.  Certificate C1 proves existence and local uniqueness of an exact zero of the event equations defining the periodic orbit.  Certificate C2 proves that, along each open segment, the intended greedy policy is unique and follows \eqref{eq:policy-sequence}.  Certificate C3 proves one-way transversality at every switching surface and shows that the incoming and outgoing vector fields cross with the same nonzero normal sign; hence the Filippov convexification creates no local sliding branch near the orbit.  Certificate C4+ verifies a positive-radius Poincare map on a two-dimensional return section with Lipschitz constant strictly less than one and image strictly inside the section box.  The contraction mapping theorem and standard Poincare-map stability then imply that $C$ is a strict local attractor for the hybrid flow, and the C3 no-sliding conclusion transfers the attraction statement to the local differential inclusion \eqref{eq:DI}.
\end{proof}

\begin{theorem}[Positive-probability nonconvergence]\label{thm:positive-prob-nonconv}
For the MDP in Appendix~\ref{app:data}, there exists $N<\infty$ and an initial condition in the local basin of the certified cycle $C$ such that the original scalar-stepsize asynchronous Monte Carlo recursion \eqref{eq:async-recursion}, using the bounded unbiased estimator \eqref{eq:geometric-estimator} and stepsizes
\[
    \gamma_t=\frac{1}{t+N},
\]
fails to converge with positive probability.
\end{theorem}

The proof is given in Section~\ref{sec:stochastic-lift}.  The result is a statement about the original stochastic recursion, not only about the deterministic mean field.

\section{Certificate summary}\label{sec:certificate-summary}

This section states the certificate structure without the full numerical tables.  All numerical details are preserved in Appendix~\ref{app:certificates}, and the support scripts are described in Appendix~\ref{app:scripts}.

\paragraph{Event equations.}
For $y=(x_0,z_0,\ldots,z_5)$ define
\[
    x_{k+1}=\Phi_{\mu_k}(x_k,z_k),
    \qquad k=0,\ldots,5,
\]
with indices read modulo six.  The exit guards are
\[
    s_0=2,\quad s_1=1,\quad s_2=3,\quad s_3=2,\quad s_4=3,\quad s_5=1.
\]
The event equations are
\[
    d_{s_k}(x_{k+1})=0,
    \qquad k=0,\ldots,5,
\]
together with the return condition $x_6=x_0$.  Equivalently, define $G:\R^9\to\R^9$ by
\[
G(y)=\begin{pmatrix}
 d_2(x_1)\\
 d_1(x_2)\\
 d_3(x_3)\\
 d_2(x_4)\\
 d_3(x_5)\\
 d_1(x_6)\\
 x_6-x_0
\end{pmatrix}.
\]
Because \eqref{eq:Phi} involves only $z,z^9,z^{40}$ and all MDP data are exact rational decimals, $G(y)=0$ is a polynomial system over $\mathbb Q$.

\paragraph{C1: orbit existence.}
A Krawczyk certificate on a radius-$10^{-30}$ box around the printed approximate solution proves that the event system has a unique exact zero in that box.

\paragraph{C2: guard signs.}
A Bernstein-polynomial guard-sign certificate proves that the sign pattern of $(d_1,d_2,d_3)$ on the six open segments is exactly
\[
\begin{array}{c|c|c}
 k&\mu_k&\operatorname{sign}(d_1,d_2,d_3)\\
\hline
0&(0,1,1)&(+,-,-)\\
1&(0,0,1)&(+,+,-)\\
2&(1,0,1)&(-,+,-)\\
3&(1,0,0)&(-,+,+)\\
4&(1,1,0)&(-,-,+)\\
5&(1,1,1)&(-,-,-).
\end{array}
\]
Thus the orbit follows the intended greedy policies and hits the intended exit guards in order.

\paragraph{C3: one-way transversality and no sliding.}
At each switching point, the incoming and outgoing vector fields have normal components with the same nonzero sign.  Therefore every convex combination of the adjacent fields also crosses the switching surface in the same direction.  This rules out local Filippov sliding near the certified orbit.

\paragraph{C4+: positive-radius Poincare contraction.}
The return section is the switching surface
\[
    \Sigma=\{x:d_1(x)=0\}
\]
through the initial point $x_0^*$.  On a section box of radius
\[
    r_\Sigma=10^{-3},
\]
the interval verifier proves
\[
    \sup_{(u,v)\in[-10^{-3},10^{-3}]^2}\|D\Pi_\Sigma(u,v)\|_\infty
    \le 0.623283830439181<1,
\]
and
\[
    \Pi_\Sigma(V_\Sigma(10^{-3}))\subset\interior(V_\Sigma(10^{-3})).
\]
The certified invariant margin is at least
\[
    m_\Sigma\ge 3.76\times 10^{-4}.
\]
The verifier also logs a strict sign-bracketing certificate for every event root on every section subbox: at the two endpoints of the final root interval, the guard polynomial has opposite strict interval signs, while $\partial_r\phi_k$ is bounded away from zero.  Thus each event root exists and is unique throughout the certified section box.  At the same radius, the section verifier certifies the positive-radius itinerary: the incoming and outgoing normal velocities have the same nonzero sign at every crossing, with outgoing lower bound at least $3.23$, and degree-$40$ Bernstein first-hit guard-sign checks prove that the intended guard is the first hit on every segment.  The active exit guard is certified positive on the pre-bracket interval $[R_k^+,1]$, while the non-exit guards are certified on $[R_k^-,1]$ with only the intended entry endpoint zero allowed.  The minimum Bernstein lower bound in these first-hit checks is $9.89\times10^{-10}$, and the minimum signed non-active guard lower bound at the crossing boxes is $1.455\times10^{-1}$.  Consequently the C4+ map is the true oriented local hybrid return map on $V_\Sigma(10^{-3})$.

\section{Stochastic lift}\label{sec:stochastic-lift}

We now prove Theorem~\ref{thm:positive-prob-nonconv}.  The deterministic ingredient is a local hybrid trapping theorem for perturbed Euler paths.  Its detailed proof is given in Appendix~\ref{app:hybrid-proof}; the statement is included here because it is the bridge from the certified mean-field cycle to the stochastic recursion.

\begin{theorem}[Hybrid trapping under small cumulative perturbations]\label{thm:hybrid-trapping}
Let $C$ be a periodic orbit satisfying C1--C4+.  Then there exist a compact tube $U$ around $C$, two disjoint open sets $U_a,U_b\subset U$ intersecting two separated segments of $C$, and constants
\[
    \eta>0,
    \qquad
    \bar\gamma>0,
\]
with the following property.  Let $Y_t$ satisfy
\begin{equation}\label{eq:perturbed-euler}
    Y_{t+1}=Y_t+\gamma_t\bigl(h(Y_t)+\xi_{t+1}\bigr),
\end{equation}
where $h(J)=D(J^{\mu(J)}-J)$ away from ties and the policy choice agrees with the unique greedy policy whenever $Y_t$ is not on a tie surface.  Assume
\[
    0<\gamma_t\le\bar\gamma,
    \qquad
    \gamma_t\to0,
    \qquad
    \sum_t\gamma_t=\infty,
\]
and
\begin{equation}\label{eq:small-cumulative-perturbation}
    \sup_{m\ge0}\left\|\sum_{t=0}^{m}\gamma_t\xi_{t+1}\right\|_\infty\le\eta.
\end{equation}
If $Y_0$ is chosen in a sufficiently small post-crossing neighborhood of the section $V$, then the entire path remains in $U$ and visits both $U_a$ and $U_b$ infinitely often.  In particular, $Y_t$ does not converge.
\end{theorem}

\begin{proof}[Proof of Theorem~\ref{thm:positive-prob-nonconv}]
Write the recursion as
\[
    J_{t+1}=J_t+\gamma_t(h(J_t)+M_{t+1}),
\]
where
\[
    h(J_t)=D(J^{\mu_t}-J_t)
\]
and
\[
    M_{t+1}=e_{I_t}\bigl(R_t-J_t(I_t)\bigr)-D(J^{\mu_t}-J_t).
\]
By the conditional independence assumptions and \eqref{eq:mean-drift},
\[
    \E[M_{t+1}\mid\mathcal F_t]=0.
\]
Let $U$ and $\eta$ be supplied by Theorem~\ref{thm:hybrid-trapping}, and define the exit time
\[
    \tau_U=\inf\{t\ge0:J_t\notin U\}.
\]
Since $U$ is compact, $\|J\|_\infty$ is bounded on $U$.  By Lemma~\ref{lem:geom-estimator}, $|R_t|\le G_{\max}/(1-\alpha)$ almost surely.  Also $\|D(J^{\mu_t}-J_t)\|_\infty$ is bounded on $U$, because there are finitely many policies and each $J^\mu$ is finite.  Hence there exists $B<\infty$ such that, for all $t<\tau_U$,
\begin{equation}\label{eq:bounded-martingale-increments}
    \|M_{t+1}\|_\infty\le B.
\end{equation}
Define the stopped martingale
\[
    S_m=\sum_{t=0}^{m}\gamma_tM_{t+1}\one_{\{t<\tau_U\}}.
\]
For each coordinate $i$, Doob's maximal inequality and \eqref{eq:bounded-martingale-increments} imply
\[
    \Prob\left(\sup_{m\ge0}|S_m(i)|>\eta\right)
    \le \frac{B^2}{\eta^2}\sum_{t=0}^{\infty}\gamma_t^2.
\]
A union bound over the three coordinates yields
\[
    \Prob\left(\sup_{m\ge0}\|S_m\|_\infty>\eta\right)
    \le \frac{3B^2}{\eta^2}\sum_{t=0}^{\infty}\gamma_t^2.
\]
For $\gamma_t=(t+N)^{-1}$,
\[
    \sum_{t=0}^{\infty}\gamma_t^2\le \frac{1}{N-1}.
\]
Choose $N$ large enough so that
\[
    \frac{3B^2}{\eta^2(N-1)}<\frac12
\]
and also so that $\gamma_0\le\bar\gamma$.  Then
\[
    \Prob\left(\sup_{m\ge0}\|S_m\|_\infty\le\eta\right)>\frac12.
\]
On this positive-probability event, the cumulative perturbation condition \eqref{eq:small-cumulative-perturbation} holds for the stopped process.  Apply Theorem~\ref{thm:hybrid-trapping} to the stopped recursion.  If $\tau_U<\infty$, then the stopped and original recursions agree up to the first exit time, while the trapping theorem keeps the stopped path inside $U$ for all times, including at $\tau_U$; this contradicts the definition of $\tau_U$.  Hence $\tau_U=\infty$.  Therefore the original recursion remains in $U$ and visits the two disjoint open sets $U_a$ and $U_b$ infinitely often.  It cannot converge.
\end{proof}

\begin{remark}[Stepsizes in finite numerical illustrations]\label{rem:illustration-stepsizes}
The proof uses $\gamma_t=1/(t+N)$ because it gives a clean Robbins--Monro martingale estimate.  This choice is mathematically valid but poor for short simulations.  The ODE-time accumulated by iteration $m$ is approximately
\[
    \sum_{t=0}^{m}\frac{1}{t+N}\approx \log\frac{m+N}{N}.
\]
Since the certified period is about $33.18$, observing one full period with a large $N$ would require roughly
\[
    m\approx N(e^{33.18}-1),
\]
which is astronomically large.  Numerical illustrations should therefore use a slower square-summable decay such as $\gamma_t=c(t+N)^{-0.55}$, which still satisfies $\sum_t\gamma_t=\infty$ and $\sum_t\gamma_t^2<\infty$.
\end{remark}

\section{Optimal value, mechanism, and separation from the cycle}\label{sec:optimality-mechanism}

For $\mu^*=(0,1,1)$, exact computation gives
\[
    J^{\mu^*}=(-85.8269548750,-92.9772951300,-93.9064262591).
\]
At this point,
\[
    d(J^{\mu^*})=(6.3364556585,-13.6208389862,-5.4965724948).
\]
Therefore $\mu^*$ is greedy with respect to its own value.  Hence
\[
    T_{\mu^*}J^{\mu^*}=TJ^{\mu^*}=J^{\mu^*}.
\]
The discounted Bellman fixed point is unique, so $J^{\mu^*}=J^*$.  The certified orbit contains
\[
    x_0\approx(-8.5936554911,-0.0178290153,0.4941993842),
\]
and therefore
\[
    \|x_0-J^*\|_\infty>90.
\]
Thus the nonconvergent behavior is not convergence to the optimal value.

The difference between uniform and nonuniform sampling is also visible in residual variables.  If sampling is uniform, then $D=cI$.  Inside a fixed-policy region,
\[
    \dot J=c(J^\mu-J),
\]
so
\[
    J(t)=J^\mu+e^{-ct}(J(0)-J^\mu),
\]
a scalar straight-line relaxation.  In residual variables $X=T_\mu J-J$,
\[
    J^\mu-J=(I-\alpha P_\mu)^{-1}X,
\]
and therefore
\[
    \dot X=-(I-\alpha P_\mu)D(I-\alpha P_\mu)^{-1}X.
\]
If $D=cI$, this reduces to
\[
    \dot X=-cX.
\]
For nonuniform $D$, the matrix
\[
    (I-\alpha P_\mu)D(I-\alpha P_\mu)^{-1}
\]
need not preserve the order structure of the residual.  This is the mechanism that makes the certified switched cycle possible.

\section{Discussion and future directions}\label{sec:discussion}

The counterexample resolves the scalar-stepsize nonuniform question negatively, but it does not give a complete classification of when nonuniform asynchronous Monte Carlo OPI converges or fails.  The natural next problem is to decide, from the MDP and update distribution, whether the diagonal distortion $D$ can generate a stable switched cycle or whether the residual dynamics retain enough monotonicity to force convergence.

\paragraph{A sufficient condition for mildly nonuniform sampling.}
Let
\[
    \sigma=\frac{\max_i q_i}{\min_i q_i}
\]
be the condition number of the update distribution.  A Tsitsiklis-style residual argument can be adapted when
\begin{equation}\label{eq:mild-nonuniform-condition}
    \sigma<\frac{1}{\alpha}.
\end{equation}
The proof sketch is as follows.  Let $X_t=TJ_t-J_t$ and let $c_t=\|X_t\|_\infty$.  When $\Gamma_t=\gamma_tD$, the nonuniform update matrix no longer commutes with $P_{\mu_t}$.  The residual recursion acquires a commutator term of the form
\[
    R_t=\alpha\Gamma_t^{-1}[P_{\mu_t},\Gamma_t](J_t-J^{\mu_t}).
\]
Since
\[
    \left|\frac{q_j}{q_i}-1\right|\le \sigma-1
\]
and the usual residual estimate gives
\[
    \|J_t-J^{\mu_t}\|_\infty\le \frac{c_t}{1-\alpha},
\]
one obtains
\[
    |R_t(i)|\le \beta c_t,
    \qquad
    \beta=\frac{\alpha(\sigma-1)}{1-\alpha}.
\]
Condition \eqref{eq:mild-nonuniform-condition} is exactly $\beta<1$.  The residual recursion then has the same qualitative form as in Tsitsiklis's proof, except for a contraction term $\beta\|X_t\|_\infty\one$ and a martingale noise term.  This yields $\limsup_t X_t\le0$ and the remaining convergence argument follows the uniform case.

This condition is very restrictive when $\alpha$ is close to one.  In the counterexample, $\alpha=0.95$ and $q=(0.02,0.18,0.80)$, so $\sigma=40$, far outside the sufficient range $\sigma<1/\alpha\approx1.0526$.  The condition therefore explains a safe mildly nonuniform regime but does not characterize the sharp boundary.

\paragraph{Open problems.}
Several questions remain open.
\begin{enumerate}[leftmargin=2em]
\item Is there a useful necessary-and-sufficient condition, or a computable sufficient condition, that rules out attracting policy-switching cycles for a given pair $(P,g,D)$?
\item Can the sufficient condition \eqref{eq:mild-nonuniform-condition} be substantially weakened by using matrix measures, weighted norms, or policy-dependent Lyapunov functions instead of the uniform sup-norm residual estimate?
\item Are component-dependent stepsizes, action-uniform update schemes, or normalization by empirical update frequencies enough to recover convergence in all discounted tabular settings, or do further pathologies remain?
\item Can the certified-counterexample search be automated to map the transition from convergence to nonconvergence as the skewness of $q$ varies?
\item Does an analogous diagonal-distortion mechanism appear in action-value MC-O-PI, approximate policy iteration, or policy-gradient algorithms with nonuniform sampling?
\end{enumerate}

The present paper should therefore be read as a certified obstruction result rather than a complete theory of nonuniform OPI.  It identifies a failure mode that any general convergence theorem for nonuniform asynchronous OPI must avoid.

\appendix

\section{Counterexample data}\label{app:data}

The discount factor is
\[
    \alpha=0.95=\frac{19}{20}.
\]
All decimals below are interpreted as exact rational numbers.  The costs are
\[
    g=\begin{pmatrix}
    2.8645365128 & 6.2604856687\\
    8.8765458733 & -6.6644477357\\
    -4.4428868652 & -7.5034828348
    \end{pmatrix}.
\]
The transition rows are
\[
\begin{aligned}
P(1,0)&=(0.0083421983,\;0.5161392971,\;0.4755185046),\\
P(1,1)&=(0.4115295055,\;0.3414919270,\;0.2469785675),\\
P(2,0)&=(0.0176113658,\;0.9549690854,\;0.0274195488),\\
P(2,1)&=(0.3574261643,\;0.1754143753,\;0.4671594604),\\
P(3,0)&=(0.6791060489,\;0.0358694514,\;0.2850244997),\\
P(3,1)&=(0.3221209528,\;0.3803472570,\;0.2975317902).
\end{aligned}
\]
Each row sums exactly to one.

The approximate initial point and event parameters are
\[
\bar x_0=\begin{pmatrix}
-8.59365549110427684060547352373810356279806687\\
-0.0178290152994757047995861123714012318312941336\\
0.494199384159623810384986644320237043033461336
\end{pmatrix}
\]
and
\[
\begin{aligned}
\bar z_0&=0.747518894957396298068037965192410779744158632,\\
\bar z_1&=0.983925376878883085999540632135004996246384770,\\
\bar z_2&=0.978860491445977974417005488978870175248391989,\\
\bar z_3&=0.957246782943564021435154886545707794901486695,\\
\bar z_4&=0.748005437128185097605270137358697189429422619,\\
\bar z_5&=0.999089891635080634085860840337534558610691943.
\end{aligned}
\]

\section{Detailed finite certificates}\label{app:certificates}

\subsection{C1: Krawczyk certificate for the orbit}\label{app:C1}

Let $\bar y=(\bar x_0,\bar z_0,\ldots,\bar z_5)$ and let
\[
    \mathcal B=\bar y+[-10^{-30},10^{-30}]^9.
\]
The verifier \path{verify_c1_c3_certificates.py} parses every decimal datum as an exact rational number.  It evaluates $G(\bar y)$ and $DG(\bar y)$ exactly over $\mathbb Q$, constructs a fixed 80-digit decimal-rational preconditioner $A$ for the Krawczyk map, and interval-encloses $DG(\mathcal B)$ by forward automatic differentiation with outward-rounded interval arithmetic.

The Krawczyk operator used by the verifier is
\[
    K(\bar y,\mathcal B)=\bar y-AG(\bar y)+(I-ADG(\mathcal B))(\mathcal B-\bar y).
\]
The certificate log prints
\[
    \|G(\bar y)\|_\infty
    \le 7.074240613034390477\times10^{-43},
\]
\[
    \|A\|_\infty
    \le 12.28569267773200443,
\]
\[
    \|I-ADG(\mathcal B)\|_\infty
    \le 2.303711148043807579\times10^{-25},
\]
and therefore
\[
\begin{aligned}
\operatorname{rad}K(\bar y,\mathcal B)
&\le \|A\|_\infty\|G(\bar y)\|_\infty
    +\|I-ADG(\mathcal B)\|_\infty10^{-30}\\
&\le 8.691194610007328100\times10^{-42}<10^{-30}.
\end{aligned}
\]
Consequently,
\[
    K(\bar y,\mathcal B)\subset\interior(\mathcal B).
\]
The Krawczyk theorem gives a unique exact zero $y^*\in\mathcal B$ of $G$.

\subsection{C2: guard-sign certificate}\label{app:C2}

For segment $k$, write
\[
    J(r)=\Phi_{\mu_k}(x_k,r),\qquad r\in[z_k,1].
\]
Each guard polynomial has the form
\[
    d_i(J(r))=a+br+cr^9+er^{40}.
\]
Set
\[
\sigma_{k,i}=\begin{cases}
+1,&\mu_k(i)=0,\\
-1,&\mu_k(i)=1.
\end{cases}
\]
The required condition is
\[
    \sigma_{k,i}d_i(J(r))\ge0
\]
throughout the segment.  Substitute
\[
    r=z_k+(1-z_k)u,
    \qquad u\in[0,1],
\]
and express the signed polynomial in degree-$40$ Bernstein form:
\[
    \sigma_{k,i}d_i(J(z_k+(1-z_k)u))
    =\sum_{m=0}^{40}b_m^{(k,i)}\binom{40}{m}u^m(1-u)^{40-m}.
\]
Because the Bernstein basis is nonnegative on $[0,1]$, interval lower bounds $b_m^{(k,i)}\ge0$ certify nonnegativity of the signed guard on the whole segment.  The verifier checks these inequalities over the full box $\mathcal B$.  The only zero coefficients are the intended endpoint zeros forced by $G(y^*)=0$.  All other coefficients have the following certified lower bounds:
\[
\begin{array}{c|c|c|c}
 k&\mu_k&\text{guard}&\text{lower bound for non-endpoint Bernstein coefficients}\\
\hline
0&(0,1,1)&d_1&5.86\\
0&(0,1,1)&d_2&0.22\\
0&(0,1,1)&d_3&0.14\\
1&(0,0,1)&d_1&0.51\\
1&(0,0,1)&d_2&0.85\\
1&(0,0,1)&d_3&11.74\\
2&(1,0,1)&d_1&0.68\\
2&(1,0,1)&d_2&20.53\\
2&(1,0,1)&d_3&0.26\\
3&(1,0,0)&d_1&19.77\\
3&(1,0,0)&d_2&0.66\\
3&(1,0,0)&d_3&0.51\\
4&(1,1,0)&d_1&0.46\\
4&(1,1,0)&d_2&10.53\\
4&(1,1,0)&d_3&0.13\\
5&(1,1,1)&d_1&0.011\\
5&(1,1,1)&d_2&17.60\\
5&(1,1,1)&d_3&0.0036.
\end{array}
\]
The same verifier prints the global certified lower bound
\[
    \min_{k,i,m\text{ non-endpoint}}b_m^{(k,i)}
    \ge 3.675551493546090391\times10^{-3}.
\]

\subsection{C3: one-way transversality and no Filippov sliding}\label{app:C3}

Let
\[
    n_k=\nabla d_{s_k}(x_{k+1}),
    \qquad
    f_k=f_{\mu_k},
    \qquad
    f_{k+1}=f_{\mu_{k+1}},
\]
where indices are read modulo six.  The notation $f_{k+1}(x_{k+1})$ means the outgoing policy vector field evaluated at the same switching point $x_{k+1}$; it does not mean evaluation at $x_{k+2}$.  Incoming transversality requires
\[
    n_k^\top f_k(x_{k+1})\ne0.
\]
For the differential inclusion \eqref{eq:DI}, we also require the outgoing vector field to cross the same switching surface in the same direction:
\[
    n_k^\top f_{k+1}(x_{k+1})\ne0,
    \qquad
    \sgn(n_k^\top f_k(x_{k+1}))=\sgn(n_k^\top f_{k+1}(x_{k+1})).
\]
The verifier gives the following interval enclosures:
\[
\begin{array}{c|c|c|c}
 k&\text{guard}&n_k^\top f_k(x_{k+1})&n_k^\top f_{k+1}(x_{k+1})\\
\hline
0&d_2=0&[0.5300,0.5301]&[42.6155,42.6156]\\
1&d_1=0&[-25.0820,-25.0819]&[-25.9101,-25.9100]\\
2&d_3=0&[9.6649,9.6651]&[9.5711,9.5713]\\
3&d_2=0&[-11.8302,-11.8300]&[-33.4407,-33.4406]\\
4&d_3=0&[-0.3103,-0.3101]&[-3.2309,-3.2308]\\
5&d_1=0&[10.0605,10.0607]&[18.5810,18.5811].
\end{array}
\]
Thus, in a sufficiently small neighborhood of each switching point, every vector in the convex hull of the adjacent fields has normal component bounded away from zero with the same sign.  Hence the Filippov inclusion has no sliding branch on the switching surface near the certified orbit; all solutions cross the surface one-way.

\subsection{C4+: positive-radius Poincare-section contraction certificate}\label{app:C4}

The initial and final section is the switching surface
\[
    \Sigma=\{x:d_1(x)=0\}
\]
through $x_0^*$.  This hyperplane is crossed twice during one geometric period: at $x_0^*$ the time-oriented crossing is from $d_1<0$ to $d_1>0$ and the outgoing policy is $\mu_0$, while the later crossing at $x_2^*$ has the opposite orientation and a different policy phase.  The map $\Pi_\Sigma$ below is the oriented return map that starts on $\Sigma$ with the outgoing branch $\mu_0$ and follows the six-event itinerary; consequently its first return from $x_0^*$ is the full-period point $x_6^*=x_0^*$, not the intermediate oppositely oriented crossing.  Let
\[
    n_0=\nabla d_1=\alpha(P(1,1,\cdot)-P(1,0,\cdot)).
\]
Numerically,
\[
    n_0=(0.38302794184,-0.165915001595,-0.217112940245),
\]
so $n_{0,1}\ne0$.  We parametrize $\Sigma$ by
\[
    x_2=x_{0,2}^*+u,
    \qquad
    x_3=x_{0,3}^*+v,
\]
\[
    x_1=x_{0,1}^*-
    \frac{n_{0,2}u+n_{0,3}v}{n_{0,1}}.
\]
Equivalently,
\[
    x=x_0^*+B\binom{u}{v},
\]
where
\[
    B=\begin{pmatrix}
    -n_{0,2}/n_{0,1} & -n_{0,3}/n_{0,1}\\
    1&0\\
    0&1
    \end{pmatrix}
    \approx
    \begin{pmatrix}
    0.4331668358&0.5668331642\\
    1&0\\
    0&1
    \end{pmatrix}.
\]
The certified section box is
\[
    V_\Sigma(r_\Sigma)
    =\{x_0^*+B(u,v)^\top:|u|\le r_\Sigma,\ |v|\le r_\Sigma\},
    \qquad r_\Sigma=10^{-3}.
\]
The verifier uses outward-rounded interval arithmetic with 100 decimal digits.  The section box $[-10^{-3},10^{-3}]^2$ is subdivided into a $4\times4$ grid.  On each subbox, the verifier isolates the six event roots by interval Newton iteration applied to
\[
    \phi_k(x,r)=d_{s_k}(\Phi_{\mu_k}(x,r)),
    \qquad
    \Phi_{\mu_k}(x,r)=J^{\mu_k}+(x-J^{\mu_k})\odot(r,r^9,r^{40}).
\]
For fixed $x$, the guard polynomial has the form
\[
    \phi_k(x,r)=a(x)+b(x)r+c(x)r^9+e(x)r^{40}.
\]
For every section subbox and every segment $k$, the verifier checks strict sign-bracketing at the two endpoints of the certified root interval $R_k=[\ell_k,u_k]$,
\[
    \phi_k(X_k,\ell_k)<0<\phi_k(X_k,u_k)
    \quad\text{or}\quad
    \phi_k(X_k,u_k)<0<\phi_k(X_k,\ell_k),
\]
together with
\[
    0\notin \partial_r\phi_k(X_k,R_k)
\]
and
\[
    0\notin n_k^\top f_{\mu_k}(Y_k),
    \qquad
    Y_k=\Phi_{\mu_k}(X_k,R_k),
\]
so each event root exists uniquely in $R_k$ for every point in the section subbox, and the event derivative is well-defined throughout the certified box.  The verifier also evaluates the outgoing field at the same crossing box and checks
\[
    0\notin n_k^\top f_{\mu_{k+1}}(Y_k),
    \qquad
    \operatorname{sgn}\bigl(n_k^\top f_{\mu_k}(Y_k)\bigr)
    =
    \operatorname{sgn}\bigl(n_k^\top f_{\mu_{k+1}}(Y_k)\bigr),
\]
with indices read modulo six.  Hence the no-sliding and same-side crossing condition holds on the full certified section tube, not only on the radius-$10^{-30}$ orbit box.

Finally, the same section verifier certifies the first-hit guard order on each segment.  Write the certified root bracket as $R_k=[R_k^-,R_k^+]$.  Since the true exit root lies inside this bracket, the verifier checks the active exit guard on the pre-bracket interval $r\in[R_k^+,1]$ with no excluded Bernstein coefficient.  It checks the non-exit guards on the larger interval $r\in[R_k^-,1]$, allowing only the intended entry guard's endpoint zero at $r=1$.  All these degree-$40$ Bernstein lower bounds are strictly positive, with global minimum $9.89\times10^{-10}$.  Thus no non-intended guard can be hit before the prescribed exit bracket anywhere in $V_\Sigma(10^{-3})$, and the non-active guards also have the required signs at the crossing boxes.  The event derivative is enclosed by
\[
D\Psi_k(x)=
\left[
I-\frac{f_{\mu_k}(y)n_k^\top}{n_k^\top f_{\mu_k}(y)}
\right]
\diag(r,r^9,r^{40}),
\qquad y=\Phi_{\mu_k}(x,r).
\]
The six event derivatives are multiplied and converted to section coordinates:
\[
    D\Pi_\Sigma(u,v)=S\,D\Pi(x_0^*+B(u,v)^\top)B,
\]
where
\[
    S=\begin{pmatrix}0&1&0\\0&0&1\end{pmatrix}
\]
selects the output coordinates $(x_2-x_{0,2}^*,x_3-x_{0,3}^*)$.

The interval verifier prints the following global certificate over all 16 section subboxes:
\[
\begin{array}{c|c}
\text{quantity}&\text{certified value/bound}\\
\hline
\text{section radius }r_\Sigma&10^{-3}\\
\text{center section residual }\|\Pi_\Sigma(0,0)\|_\infty&7.08\times10^{-43}\\
\text{maximum event-root interval width}&2.01\times10^{-5}\\
\text{strict endpoint bracketing for every event}&\text{true}\\
\min|\partial_r\phi_k|&20.73\\
\min|n_k^\top f_{\mu_k}(Y_k)|&0.310\\
\min|n_k^\top f_{\mu_{k+1}}(Y_k)|&3.2308\\
\text{incoming/outgoing same nonzero sign on the section box}&\text{true}\\
\text{first-hit guard order certified on the section box}&\text{true}\\
\text{minimum first-hit Bernstein lower bound}&9.89\times10^{-10}\\
\text{minimum crossing non-active signed-guard lower bound}&0.1455\\
\text{row 1 absolute row-sum of }D\Pi_\Sigma&0.623283830439181\\
\text{row 2 absolute row-sum of }D\Pi_\Sigma&0.439294078059216.
\end{array}
\]
Consequently,
\[
    \sup_{(u,v)\in[-10^{-3},10^{-3}]^2}
    \|D\Pi_\Sigma(u,v)\|_\infty
    \le 0.623283830439181<1.
\]
The mean-value theorem in section coordinates gives, for every $(u,v)\in[-10^{-3},10^{-3}]^2$,
\[
\begin{aligned}
    \|\Pi_\Sigma(u,v)\|_\infty
    &\le\|\Pi_\Sigma(0,0)\|_\infty
    +0.623283830439181\|(u,v)\|_\infty\\
    &\le 7.08\times10^{-43}+0.623283830439181\cdot10^{-3}<10^{-3}.
\end{aligned}
\]
Thus
\[
    \Pi_\Sigma(V_\Sigma(10^{-3}))\subset\interior(V_\Sigma(10^{-3})),
\]
and the return map is a contraction on the section box with certified Lipschitz constant
\[
    \kappa_\Sigma=0.623283830439181<1.
\]
The resulting positive invariant margin is at least
\[
    m_\Sigma
    =10^{-3}-\left(7.08\times10^{-43}+0.623283830439181\cdot10^{-3}\right)
    >3.76\times10^{-4}.
\]

\section{Proof of the hybrid trapping theorem}\label{app:hybrid-proof}

\begin{proof}[Proof of Theorem~\ref{thm:hybrid-trapping}]
We give the proof in four steps.

\medskip
\noindent\textbf{Step 1: finite-time shadowing inside one policy region.}
Fix a compact neighborhood $K$ of $C$ small enough that, except in disjoint neighborhoods of the six switching points, the active policy is unique and equal to the intended one.  On each such region,
\[
    h(J)=f_k(J)=D(J^{\mu_k}-J)
\]
is affine.  Hence there is $L<\infty$ such that
\[
    \|f_k(x)-f_k(y)\|_\infty\le L\|x-y\|_\infty
\]
for all relevant $x,y\in K$ and all $k$.

Let $\varphi_k(t,x)$ denote the exact flow of $\dot x=f_k(x)$.  Consider a perturbed Euler path that remains in the same region during the ODE-time interval $[0,T]$, where $T$ is bounded by a constant larger than the longest certified segment time.  Define ODE time
\[
    s_m=\sum_{t=t_0}^{m-1}\gamma_t.
\]
Standard discrete Gronwall gives, uniformly for $s_m\le T$,
\begin{equation}\label{eq:finite-shadow}
    \|Y_m-\varphi_k(s_m,Y_{t_0})\|_\infty
    \le C_1\left(\max_{t_0\le t<m}\gamma_t+
    \sup_{t_0\le r<m}\left\|\sum_{t=t_0}^{r}\gamma_t\xi_{t+1}\right\|_\infty\right).
\end{equation}
Here $C_1$ depends only on $K,L,T$ and the bound on $\|f_k\|$ on $K$.  The $\max\gamma_t$ term is the Euler truncation error; the second term is the cumulative additive perturbation.  Inequality \eqref{eq:finite-shadow} is obtained by subtracting the Euler scheme from the exact variation-of-constants formula and applying the discrete Gronwall inequality.

\medskip
\noindent\textbf{Step 2: stability of guard hitting.}
For the nominal $k$th segment, let the exit guard be $d_{s_k}=0$ and let the nominal exit time be $\tau_k$.  Define
\[
    F_k(t,x)=d_{s_k}(\varphi_k(t,x)).
\]
At the certified hit point,
\[
    \partial_tF_k(\tau_k,x_k^*)
    =\nabla d_{s_k}(x_{k+1}^*)^\top f_k(x_{k+1}^*)\ne0.
\]
By C3 this derivative is bounded away from zero.  Hence the implicit function theorem gives a $C^1$ hitting-time map $\theta_k(x)$ near $x_k^*$ such that
\[
    F_k(\theta_k(x),x)=0.
\]
Moreover there is $C_2<\infty$ such that
\begin{equation}\label{eq:hitting-lip}
    |\theta_k(x)-\theta_k(x')|+
    \|\varphi_k(\theta_k(x),x)-\varphi_k(\theta_k(x'),x')\|_\infty
    \le C_2\|x-x'\|_\infty.
\end{equation}
The guard-sign certificate C2 gives positive margins to all nonactive guards near the orbit, so after shrinking $K$ the intended guard is the first guard hit.  For the concrete section box used in Step~3, the C4+ section verifier makes this quantitative on the whole box $V_\Sigma(10^{-3})$: on every section subbox and every segment, the active exit guard has a positive Bernstein lower bound on the pre-bracket interval $[R_k^+,1]$, the non-exit guards have positive Bernstein lower bounds on $[R_k^-,1]$ after excluding only the intended entry endpoint zero, the global first-hit Bernstein lower bound is at least $9.89\times10^{-10}$, the non-active guard signs at crossing boxes have lower bound at least $0.1455$, and the incoming/outgoing normal velocities have the same nonzero signs.  Hence the scale used in the contraction estimate is also a scale on which the prescribed itinerary is certified, and the map $\Pi$ used below agrees with the true oriented hybrid return map on the full certified section box, not merely on an unspecified smaller neighborhood.

For the discrete path, \eqref{eq:finite-shadow} and \eqref{eq:hitting-lip} imply that the first sign change of $d_{s_k}$ occurs at a point within
\begin{equation}\label{eq:event-error}
    C_3\left(\max\gamma_t+
    \sup\left\|\sum\gamma_t\xi_{t+1}\right\|_\infty\right)
\end{equation}
of the exact event map.  The one-way transversality in C3 ensures that crossing the surface by a small amount places the iterate on the correct outgoing side; the Filippov convexification cannot create sliding because all adjacent normal velocities have the same nonzero sign.  In the numerical certificate, this same-sign incoming/outgoing transversality is also checked on the full section tube, with $\min |n_k^\top f_{\mu_{k+1}}(Y_k)|\ge 3.2308$.

\medskip
\noindent\textbf{Step 3: one-cycle return estimate.}
Let $\Psi_k$ be the exact event map from the $k$th section to the next one, and let
\[
    \Pi=\Psi_5\circ\cdots\circ\Psi_0
\]
be the exact Poincare map.  By applying Step 2 through six consecutive segments, there is $C_4<\infty$ such that a discrete perturbed circuit starting from a point $x\in V$ returns to the initial section at a point $\widehat\Pi(x)$ satisfying
\begin{equation}\label{eq:return-error}
    \|\widehat\Pi(x)-\Pi(x)\|_\infty
    \le C_4\left(\max\gamma_t+
    \sup\left\|\sum\gamma_t\xi_{t+1}\right\|_\infty\right),
\end{equation}
where the maximum and supremum are taken over the steps in that circuit.

C4+ is certified on the section box
\[
    V=V_\Sigma(r_\Sigma),\qquad r_\Sigma=10^{-3},
\]
and includes first-hit guard-sign and same-sign outgoing-transversality checks on that same box.  Therefore $\Pi$ is the true oriented six-event hybrid return map on $V$.  It has contraction constant
\[
    \kappa_\Sigma=0.623283830439181<1
\]
and invariant margin
\[
    m_\Sigma=\dist_\infty(\Pi(V),\partial V)\ge3.76\times10^{-4}.
\]
Choose $\eta>0$ and $\bar\gamma>0$ so small that
\begin{equation}\label{eq:invariant-box-margin}
    C_4(\bar\gamma+2\eta)<\frac12m_\Sigma.
\end{equation}
Then every perturbed return from $V$ remains in $V$.  The contraction and the positive margin absorb the $O(\bar\gamma+\eta)$ error at every cycle.

\medskip
\noindent\textbf{Step 4: infinite cycling and nonconvergence.}
Because $\sum_t\gamma_t=\infty$, the perturbed path accumulates enough ODE time to complete infinitely many circuits, provided it remains in the tube.  Step 3 proves that it does remain in the tube and returns to the section infinitely often.  Choose two disjoint open sets $U_a$ and $U_b$ around two separated points of $C$, small enough that every admissible perturbed circuit passes through both.  This is possible because the nominal orbit is nonconstant and the tracking error is bounded by the chosen tube radius.  Therefore $Y_t$ visits $U_a$ and $U_b$ infinitely often.  Since $U_a\cap U_b=\varnothing$, the sequence $Y_t$ cannot converge.
\end{proof}

\section{Verification artifact and code availability}\label{app:scripts}

The verification scripts, input data, and certificate logs are available at
\begin{center}
	\url{https://github.com/ylcoldplayer/non-uniform-opi}
\end{center}
The repository contains three Python verifiers ---
\path{verify_c1_c3_certificates.py}, \path{verify_c4plus_decimal_interval.py},
and \path{verify_c4plus_section_interval.py} --- together with the generated
\path{.log} and \path{.json} certificate files.  Every MDP and orbit datum is
parsed as an exact rational decimal; the algebra is evaluated exactly over
$\mathbb Q$ whenever the expressions are algebraic, and every enclosure uses
outward-rounded interval arithmetic.  The global \path{Decimal} precision is set
to $100$ digits and all center--radius boxes are constructed with directed
rounding, so running the scripts regenerates the numerical bounds quoted in the
text.

\paragraph{C1--C3 verifier.}
The script \path{verify_c1_c3_certificates.py} verifies C1--C3 and writes
\path{c1_c3_certificate.log} and \path{c1_c3_certificate.json}.  It imports the
exact-rational MDP data and interval-arithmetic primitives from
\path{verify_c4plus_decimal_interval.py} (described below).  It evaluates
$G(\bar y)$ and $DG(\bar y)$ exactly over $\mathbb Q$, forms a fixed
decimal-rational preconditioner $A$, and encloses $DG(\mathcal B)$ by
outward-rounded interval automatic differentiation.  The same script prints the
Krawczyk inclusion bounds for C1, the degree-$40$ Bernstein coefficient lower
bounds and their global minimum for C2, and the two-sided transversality
intervals for C3.

\paragraph{C4+ verifiers.}
The positive-radius C4+ certificate is produced by two files.  The first,
\path{verify_c4plus_decimal_interval.py}, is the core module that the C1--C3
verifier and the section verifier both import: it defines the exact-rational MDP
data, the policy-value solver, the directed-rounded \path{Decimal} interval
arithmetic, the fixed-policy flows, the guard functions, the event derivatives,
the interval-Newton root isolation, and the strict endpoint-bracketing checks.
Run on its own, it also emits the auxiliary ambient-box certificate
\path{c4plus_interval_certificate.log} and \path{c4plus_interval_certificate.json}
on the box $\|x-x_0^*\|_\infty\le10^{-8}$; this ambient certificate is diagnostic
only.

The proof-relevant certificate is produced by
\path{verify_c4plus_section_interval.py}, which verifies the positive-radius
Poincare-section contraction and writes
\path{c4plus_section_interval_certificate.log} and
\path{c4plus_section_interval_certificate.json}.  It prints the section radius
$r_\Sigma=10^{-3}$, the section parametrization, the subdivision grid, the
interval event-root enclosures, strict endpoint sign-bracketing certificates for
those roots, the incoming and outgoing event-normal denominator checks, the
same-sign no-sliding checks on the section tube, the Bernstein first-hit
guard-sign checks, the non-active guard signs at crossings, and the
section-coordinate contraction bound
\[
\sup_{(u,v)\in[-r_\Sigma,r_\Sigma]^2}\|D\Pi_\Sigma(u,v)\|_\infty
\le \kappa_\Sigma=0.623283830439181<1,
\]
together with the containment
\[
\Pi_\Sigma(V_\Sigma(r_\Sigma))\subset\interior V_\Sigma(r_\Sigma)
\]
at invariant margin $m_\Sigma\ge3.76\times10^{-4}$.  The same log also reports the
positive-radius itinerary margins $\min |n_k^\top f_{\mu_{k+1}}(Y_k)|\ge3.2308$,
the minimum first-hit Bernstein lower bound $9.89\times10^{-10}$, and the minimum
crossing non-active guard lower bound $1.455\times10^{-1}$.

\paragraph{What is certified and what is analytic.}
The finite certificates establish Proposition~\ref{prop:certified-cycle}: the
mean field has a nonconstant attracting hybrid cycle, and the certified pair
$(r_\Sigma,m_\Sigma)$ supplies the section radius and invariant margin used in the
trapping theorem.  The stochastic lift is analytic.  In particular,
Theorem~\ref{thm:hybrid-trapping} and Theorem~\ref{thm:positive-prob-nonconv} use
standard finite-time shadowing, implicit-function stability of guard hitting, and
martingale maximal inequalities.  The constants in the hybrid perturbation
theorem, such as the one-cycle amplification constant $C_4$ in
\eqref{eq:return-error}, are shown to be finite but are not numerically
evaluated; the proof only requires that $\bar\gamma$ and $\eta$ be chosen small
enough to satisfy \eqref{eq:invariant-box-margin}.

\bibliographystyle{plain}
\bibliography{references}

@inproceedings{Baird1995,
  title = {{Residual Algorithms: Reinforcement Learning with Function Approximation}},
  author = {Baird, Leemon},
  booktitle = {{Machine Learning Proceedings 1995}},
  pages = {30--37},
  year = {1995},
  organization = {Elsevier}
}

@article{BenaimHofbauerSorin2005,
  title = {{Stochastic Approximations and Differential Inclusions}},
  author = {Bena{\"i}m, Michel and Hofbauer, Josef and Sorin, Sylvain},
  journal = {{SIAM Journal on Control and Optimization}},
  volume = {44},
  number = {1},
  pages = {328--348},
  year = {2005},
  publisher = {SIAM}
}

@article{Borkar1998,
  title = {{Asynchronous Stochastic Approximations}},
  author = {Borkar, Vivek S.},
  journal = {{SIAM Journal on Control and Optimization}},
  volume = {36},
  number = {3},
  pages = {840--851},
  year = {1998},
  publisher = {SIAM}
}

@article{Chen2018,
  title = {{On the Convergence of Optimistic Policy Iteration for Stochastic Shortest Path Problem}},
  author = {Chen, Yuanlong},
  journal = {arXiv preprint arXiv:1808.08763},
  year = {2018}
}

@article{DelattreFournier2026,
  title = {{Markov Decision Processes: On the Convergence of the Monte-Carlo First-Visit Algorithm}},
  author = {Delattre, Sylvain and Fournier, Nicolas},
  journal = {{SIAM Journal on Control and Optimization}},
  volume = {64},
  number = {3},
  pages = {1672--1697},
  year = {2026},
  publisher = {SIAM}
}

@article{FaureRoth2010,
  title = {{Stochastic Approximations of Set-Valued Dynamical Systems: Convergence with Positive Probability to an Attractor}},
  author = {Faure, Mathieu and Roth, Gr{\'e}gory},
  journal = {{Mathematics of Operations Research}},
  volume = {35},
  number = {3},
  pages = {624--640},
  year = {2010},
  publisher = {INFORMS}
}

@article{Liu2021,
  title = {{On the Convergence of Reinforcement Learning with Monte Carlo Exploring Starts}},
  author = {Liu, Jun},
  journal = {Automatica},
  volume = {129},
  pages = {109693},
  year = {2021},
  publisher = {Elsevier}
}

@article{LubarsWinnickiLivesaySrikant2021,
  title = {{Optimistic Policy Iteration for MDPs with Acyclic Transient State Structure}},
  author = {Lubars, Joseph and Winnicki, Anna and Livesay, Michael and Srikant, R.},
  journal = {arXiv preprint arXiv:2102.00030},
  year = {2021}
}

@article{OliviersVinnicombe2026,
  title = {{Convergence of Monte Carlo Optimistic Policy Iteration: Beyond Uniform State-Action Updates}},
  author = {Oliviers, Octave and Vinnicombe, Glenn},
  journal = {arXiv preprint arXiv:2606.10580},
  year = {2026}
}

@article{PerkinsLeslie2012,
  title = {{Asynchronous Stochastic Approximation with Differential Inclusions}},
  author = {Perkins, Steven and Leslie, David S.},
  journal = {{Stochastic Systems}},
  volume = {2},
  number = {2},
  pages = {409--446},
  year = {2012},
  publisher = {INFORMS}
}

@book{SuttonBarto2018,
  title = {{Reinforcement Learning: An Introduction}},
  author = {Sutton, Richard S. and Barto, Andrew G.},
  edition = {Second},
  year = {2018},
  publisher = {MIT Press}
}

@article{Tsitsiklis2002,
  title = {{On the Convergence of Optimistic Policy Iteration}},
  author = {Tsitsiklis, John N.},
  journal = {{Journal of Machine Learning Research}},
  volume = {3},
  number = {Jul},
  pages = {59--72},
  year = {2002}
}

@inproceedings{Wagner2011,
  title = {{A Reinterpretation of the Policy Oscillation Phenomenon in Approximate Policy Iteration}},
  author = {Wagner, Paul},
  booktitle = {{Advances in Neural Information Processing Systems}},
  volume = {24},
  pages = {2573--2581},
  year = {2011}
}

@inproceedings{Wagner2013,
  title = {{Optimistic Policy Iteration and Natural Actor-Critic: A Unifying View and a Non-Optimality Result}},
  author = {Wagner, Paul},
  booktitle = {{Advances in Neural Information Processing Systems}},
  volume = {26},
  pages = {1592--1600},
  year = {2013}
}

@inproceedings{WangYuanShaoRoss2022,
  title = {{On the Convergence of the Monte Carlo Exploring Starts Algorithm for Reinforcement Learning}},
  author = {Wang, Che and Yuan, Shuhan and Shao, Kai and Ross, Keith W.},
  booktitle = {{International Conference on Learning Representations}},
  year = {2022}
}

\end{document}